\documentclass[10pt]{article}

\unless\ifdefined\endofdump\makeatletter\let\endofdump\relax\makeatother\fi

\usepackage[a4paper]{geometry}
\usepackage[T1]{fontenc}
\usepackage[utf8]{inputenc}
\usepackage{lmodern}
\usepackage{microtype}
\usepackage{amsmath,amssymb,amsthm}
\usepackage{booktabs}
\usepackage{array}
\usepackage{graphicx} 
\usepackage{todonotes}
\usepackage{xcolor}
\usepackage{ulem}
\usepackage{float}
\usepackage{marginnote}
\usepackage{enumerate}
\usepackage{paralist}
\usepackage{hyperref}
\usepackage{abstract}
\newcommand{\hlgreen}[1]{%
  \bgroup\markoverwith{\textcolor{addgreen}{\rule[-0.5ex]{2pt}{2.8ex}}}\ULon{#1}}
\newcommand{\hlred}[1]{%
  \bgroup\markoverwith{\textcolor{delred}{\rule[-0.5ex]{2pt}{2.8ex}}}\ULon{#1}}
\newcommand{\hlyellow}[1]{%
  \bgroup\markoverwith{\textcolor{todoyellow}{\rule[-0.5ex]{2pt}{2.8ex}}}\ULon{#1}}
\definecolor{addgreen}{HTML}{D5F5E3}   
\definecolor{delred}{HTML}{FADBD8}     
\definecolor{deltext}{HTML}{C0392B}    
\definecolor{addtext}{HTML}{1E8449}    
\definecolor{commentbg}{HTML}{D6EAF8}  
\definecolor{todoyellow}{HTML}{FFF9C4} 

\newcommand{\mcomment}[1]{\marginpar{\raggedright\scriptsize\colorbox{commentbg}{\parbox{0.95\marginparwidth}{#1}}}}
\newcommand{\ma}[2][]{\ifx\\#1\\\else\mcomment{+ #1}\fi{\color{addtext}\hlgreen{#2}}}
\newcommand{\md}[2][]{\ifx\\#1\\\else\mcomment{-- #1}\fi{\color{deltext}\hlred{#2}}}
\newcommand{\mr}[3][]{\ifx\\#1\\\else\mcomment{$\rightarrow$ #1}\fi\md{#2}\,\ma{#3}}
\newcommand{\mtodo}[2][]{\ifx\\#1\\\else\mcomment{TODO: #1}\fi\hlyellow{#2}}

\usepackage[numbers,sort&compress]{natbib}

\hypersetup{
  colorlinks=true,
  linkcolor=blue!60!black,
  urlcolor=blue!60!black,
  citecolor=blue!60!black
}

\newtheorem{theorem}{Theorem}
\newtheorem{lemma}{Lemma}

\newtheorem{remark}{Remark}

\title{Sparse Prefix Caching for Hybrid and Recurrent LLM Serving}
\author{Mikhail Shirokikh \and Sergey Nikolenko}
\date{\today}

\ifdefined\endofdump\endofdump\fi 
\begin{document}

\maketitle
\begin{abstract}
  Prefix caching is a key latency optimization for autoregressive
LLM serving, yet existing systems assume dense per-token key/value reuse.
State-space models change the structure of the problem: a recurrent layer can
resume from a single stored state rather than requiring the entire token
history. This asymmetry opens a new design point between no reuse and dense
caching: store exact recurrent states at a \emph{sparse} set of checkpoint
positions and, on a cache hit, resume from the deepest stored checkpoint and
recompute the remaining suffix exactly.

We formalize sparse prefix caching as checkpoint placement under a distribution over overlap depths, yielding an exact $O(NM)$ dynamic program. For use cases where requests share a non-trivial prefix (e.g.\ asking different questions about a single long document), we show that our method consistently improves the Pareto frontier traced by standard heuristics on real-world data. Across QuALITY and System Prompts, distribution-aware placement dominates every fixed-budget baseline on the measured layer-group Pareto frontier and matches or outperforms the strongest heuristic (block caching) while typically using substantially fewer checkpoints, with the largest gains at low checkpoint budgets where the overlap distribution is most non-uniform. The method is most relevant when many requests share a substantial but not identical prefix within a retained cache entry. It preserves exact outputs, does not change the recurrent computation itself or require new recurrent update kernels, applies to recurrent/SSM layers whose hidden state can be extracted and restored exactly, and for hybrid models can be combined with existing KV-cache compression techniques.

\end{abstract}

\section{Introduction}

Prefix caching avoids redundant prefill computation when successive requests
share a common prefix.  In Transformer-centric serving systems, the reusable state
consists of dense per-token key/value tensors, and modern runtimes exploit
paging~\citep{kwon2023pagedattention}, radix-tree
sharing~\citep{zheng2023sglang}, and distributed prompt
scheduling~\citep{srivatsa2024preble} to make this practical at scale.

Recent progress in state-space and linear recurrent
models~\citep{gu2023mamba,dao2024ssd,peng2023rwkv,yang2024gla} has
produced architectures whose recurrent layers maintain a fixed-size
hidden state rather than per-token key/value pairs. Hybrid architectures that 
combine attention with recurrent or state-space
layers~\citep{qwen2025qwen35} complicate this picture. Since a recurrent layer evolves its state as
$h_t = F(h_{t-1}, x_t)$, continuing from token~$t$ requires only $h_t$, not
the entire history of intermediate states.  Dense per-token caching of
recurrent state is therefore both unnecessary and prohibitively expensive: for
models such as Qwen-3.5, the recurrent (GatedDeltaNet) state per head scales as
$d_{\mathrm{head}}^2$, whereas attention KV storage grows only linearly in
$d_{\mathrm{head}}$ per head.

Hybrid and SSM serving do not have to choose between full dense caching and no caching at all; they can instead
store the recurrent state at a \emph{sparse} set of checkpoint positions and
recompute the gap between the deepest stored checkpoint and the divergence
point on demand. 

The resulting optimization problem is: given a memory budget of $M$ checkpoint
slots along a prefix of length~$N$, where should checkpoints be placed to
minimize expected recomputation under the overlap structure of future
requests? This question is orthogonal to the cache-admission and eviction
policies studied by Marconi~\cite{pan2024marconi}, which decide which entries to retain.
We instead ask how to make \emph{partial} prefix overlap useful for recurrent
layers by selecting which tokens to checkpoint within every single entry.

This is particularly beneficial when requests share long
but not fully matching prefixes, e.g.\ when consecutive requests share a long system
prompt, start with RAG-retrieved text, or ask different questions about the same
document. For chat-like workloads, where each consecutive request within a
conversation contains the previous one as a substring, the optimal strategy is obviously to
store only the last state, so our method does not bring additional benefits (but also incurs no overhead).

The problem is reminiscent of activation checkpointing in
training~\citep{chen2016training}, where a subset of intermediate activations
is stored during the forward pass and the rest are recomputed during
backpropagation. The key difference is that activation checkpointing
optimizes a single, known computation graph, whereas our setting involves
stochastic overlap structure across requests, shared cache capacity, and
online estimation of the overlap distribution.

\paragraph{Our contributions.}
\begin{enumerate}
  \item We prove that balanced spacing is optimal under uniform overlap and in the worst case (Theorem~\ref{thm:balanced-optimal}).
  \item We formalize sparse prefix caching as a one-sided weighted
  $k$-median problem with an exact $O(NM)$ dynamic program (Theorem~\ref{thm:dp}). 
  \item Furthermore, to enable using the empirical overlap depth distribution over
  previous requests instead of the true distribution over future overlaps,
  we show that the value function is Lipschitz in the overlap law
  (Lemma~\ref{thm:stability}) and therefore its suboptimality
  can be bounded by total variation between the empirical and true
    distributions (Theorem~\ref{thm:empirical-oracle}). We also provide a similar bound for
    the case of drift in overlap depth distributions (Theorem~\ref{thm:drift-tracking}).
  \item We validate these results on real-world datasets (QuALITY, NarrativeQA, System Prompts) and show that our
  method saves more recomputations per checkpoint than every fixed-budget baseline we consider. On QuALITY and System Prompts, the DP-derived schedule yields prototype wall-clock speedups on a representative layer group, with the largest gains at low checkpoint budgets where the overlap distribution is most non-uniform.
\end{enumerate}

The rest of the paper is organized as follows. Section~\ref{sec:related} reviews prior work on prefix caching, KV-cache compression, hybrid-model serving, and related checkpointing and facility-location formulations. Section~\ref{sec:formulation} introduces the sparse prefix caching problem and its distributional formulation. Section~\ref{sec:theory} presents the main theoretical results, including optimal balanced placement under uniform overlap, the exact dynamic program, and plug-in guarantees under stationary and drifting overlap laws. Section~\ref{sec:method} describes the experimental setup, datasets, baselines, and implementation constraints. Section~\ref{sec:results} reports token-savings and prototype wall-clock results. Section~\ref{sec:conclusion} concludes, and the appendix contains additional proofs and supplementary figures.

\section{Related Work}
\label{sec:related}

\paragraph{Serving systems and dense prefix reuse.}
The PagedAttention mechanism in vLLM demonstrated that in practice, serving LLMs depends a lot on flexible KV-cache sharing and low-fragmentation memory management~\citep{kwon2023pagedattention}. SGLang introduced RadixAttention for structured multi-call applications, allowing for automatic prefix sharing across programs~\citep{zheng2023sglang}, while Preble studied distributed prompt scheduling that co-optimizes sharing and load balancing~\citep{srivatsa2024preble}. CacheBlend addressed exact reuse of non-contiguous cached chunks as needed in retrieval-augmented generation~\citep{yao2024cacheblend}. ChunkAttention proposed a two-phase partition scheme for prefix-aware decoding that further reduces redundant
attention computation~\citep{ye2024chunkattention}. All of these systems assume dense token-indexed KV state; however; our setting is different because recurrent layers can be resumed from sparse exact checkpoints.

\paragraph{KV cache compression and selection.}
A large body of work reduces the memory footprint of dense attention caches
through eviction, quantization, or sparsification.  Scissorhands identifies
tokens that are ``pivotal'' for future attention and evicts the
rest~\citep{liu2024scissorhands}.  H$_2$O retains only a heavy-hitter set of
KV entries that accumulate the most attention
mass~\citep{zhang2024h2o}.  PyramidKV and PyramidInfer exploit the observation
that later layers need fewer cached tokens, and allocates budgets in a
layer-adaptive pyramid shape~\citep{cai2024pyramidkv,yang2024pyramidinfer}.
SnapKV clusters attention features to select representative
KV entries per layer~\citep{li2024snapkv}, and DMC proposes a learnable
policy that decides which tokens to merge, evict, or
keep~\citep{nawrot2024dmc}.

In contrast to these, our method targets recurrent states rather than KV cache,
and keeps the outputs exact.

\paragraph{Hybrid-model prefix caching.}
The closest prior work has been done by Marconi et al.~\citep{pan2024marconi}.  It identifies that hybrid models use in-place recurrent updates that make partial rollback hard, pushing serving toward exact-match entries and creating many large, low-value cache candidates. The authors propose improved admission and eviction policies. Our emphasis in this work is complementary: rather than improving the management of exact recurrent cache entries, we change the unit of reuse itself by storing sparse exact checkpoints so that partial prefix overlap becomes useful.

\paragraph{Activation checkpointing in training.}
Sparse prefix caching is structurally similar to activation checkpointing~\citep{chen2016training}, where a subset of layer activations is stored during the forward pass to reduce peak memory during backpropagation. Both problems ask where to place checkpoints along a sequential computation so as to minimize recomputation under a memory budget. The key differences are that:
\begin{inparaenum}[(i)]
\item activation checkpointing optimizes a single deterministic computation graph with known topology, whereas we optimize across requests with stochastic
overlap, and
\item in serving, checkpoint placement interacts with cache admission and eviction, which has no counterpart in training.
\end{inparaenum}

\paragraph{Facility location and $k$-median problems.}
The offline checkpoint placement problem we study is a one-sided, one-dimensional weighted $k$-median on a line. Classical $k$-median is NP-hard in general metric spaces~\citep{megiddo1984complexity}, but the one-dimensional case admits efficient dynamic programs~\citep{hassin1991improved,grover1966kmediansquare}. Our convex-hull-trick solver exploits the monotone structure of prefix sums to achieve $O(NM)$ time, which is similar to the SMAWK-based optimizations used in related interval-scheduling problems~\citep{aggarwal1987smawk}.

\section{Problem Formulation}
\label{sec:formulation}

Consider a cached prefix of length $N$ and a checkpoint set
\[
\mathcal{C} = \{c_1, c_2, \dots, c_M\}, \qquad
1 \le c_1 < \cdots < c_M \le N.
\]
If a future request overlaps the prefix up to depth $t$, then the deepest
reusable checkpoint is
\[
\ell(t; \mathcal{C}) = \max(\{0\} \cup \{c \in \mathcal{C} : c \le t\}),
\]
and exact resumption requires recomputing
\[
r(t; \mathcal{C}) = t - \ell(t; \mathcal{C})
\]
tokens of recurrent work. Let
\[
R_{\mathrm{nc}} = E[T] = \sum_{t=1}^N p_t t
\]
denote the no-cache recurrent-work baseline. We report normalized savings as
\[
\mathrm{Savings}(C)=1-\frac{E[r(T;C)]}{R_{\mathrm{nc}}}.
\]
For plots whose vertical axis is marked ``$\times$'', we additionally report the recurrent-work reduction factor
\[
\mathrm{Reduction}(C)=\frac{R_{\mathrm{nc}}}{E[r(T;C)]}
=\frac{1}{1-\mathrm{Savings}(C)}.
\]

\begin{remark}[Exactness of resuming from a checkpoint]
\label{rem:exactness}
Because the recurrent update $h_t = F(h_{t-1}, x_t)$ is deterministic,
resuming from an exact checkpoint~$h_c$ and replaying tokens $x_{c+1}, \dots,
x_t$ produces a state identical to full replay from $h_0$.  Sparse
checkpointing therefore changes only the location {where} states are stored, not the
model output. We have verified this experimentally: across all datasets, the
checkpoint-resume path produces outputs that are identical to full prefill.
\end{remark}

\paragraph{Distributional formulation.} We consider a version of
this problem where exact request structure is ignored for simplicity and the
overlap depth of each request is an i.i.d. sample from the \emph{overlap
depth distribution} $T \in \{1,\dots,N\}$ with mass $p_t = \Pr[T=t]$. 
The objective is the expected recomputation cost 
$
\mathbb{E}_T[r(T; \mathcal{C})] = \sum_{t=1}^N p_t\, (t - \ell(t;
\mathcal{C}))
$.

\paragraph{Conditioning on retained hits.}
The theoretical objective is conditional on a retained cache entry whose
prefix overlaps the incoming request. Admission/eviction and exact-match miss
handling are orthogonal system components; in the experiments below, we
combine sparse checkpoint placement with a fixed last-$K$ entry policy to
study the resulting runtime behavior.

\paragraph{Drift generalization.} We also consider a distribution drift setup,
where the overlap depth distribution changes over time. This version is
motivated by the fact that the cache hit rate changes over time for
real-world deployments, either because of changing usage scenarios~\citep{patel2024splitwise},
steadily growing context lengths~\citep{fu2024challenges}, or cache saturation~\citep{pan2024marconi}.

\paragraph{Generalizing to a trie structure.}
Real cached prefixes form a trie (prefix tree), and the truly optimal checkpoint placement would optimize over this trie structure.  However, common real-world topologies are either star-shaped (e.g., a chat with one root and many leaves) or comb-shaped (shared document with varying suffixes). In both cases, each edge of the trie is a single long shared prefix, so the trie-optimal solution decomposes into independent per-edge problems, each of which is exactly the single-prefix formulation above. We therefore restrict our attention to the single-prefix case without loss of generality for these topologies.

\section{Theoretical Analysis}
\label{sec:theory}

\subsection{Optimal checkpoint placement}

We now collect the main mathematical results about sparse checkpoint placement. Appendix~\ref{app:proofs} contains the proofs of all statements in this section. 

Throughout this section we define gap lengths
\[
g_i = c_{i+1}-c_i, \qquad i=0,\dots,M,
\]
where $c_0=0$ and $c_{M+1}=N+1$. Then the $g_i$ are positive integers summing to
$N+1$.

\paragraph{Balanced checkpointing.}We begin by studying the simplest strategy where checkpoints are
evenly spaced. We call a checkpoint set $\mathcal{C}$ \emph{balanced} if
$\forall i \neq j: |g_i - g_j| \leq~1$.

\begin{theorem}[Balanced checkpointing is optimal under uniform overlap
and in worst-case]
\label{thm:balanced-optimal}
Assume $p_t = 1/N$ for all $t=1,\dots,N$. Let $K=M+1$ and write $N+1 = qK + \rho$ with $0 \le \rho < K$. Then,
\begin{enumerate}
\item Checkpoint set minimizes $\mathbb{E}[r(T;\mathcal{C})]$ if and only if it
is {balanced}. The exact optimal expected recomputation is
\[
\mathbb{E}[r(T;\mathcal{C}^\star)]
= \frac{1}{N}
\left(
(K-\rho)\frac{q(q-1)}{2} + \rho\frac{q(q+1)}{2}
\right)
= \frac{N}{2(M+1)} + O(1),
\]
where the $O(1)$ term is for $N\to\infty$ with $M$ fixed.
\item {Balanced} sets also minimize $\max\limits_t\, r(t;\mathcal{C})$
with worst-case recomputation of $\Big\lceil\frac{N+1}{M+1}\Big\rceil-1$ tokens. 

\end{enumerate}
\end{theorem}

One notable regime is $M=O(\sqrt N)$, where
the number of checkpoints and the expected/worst-case recomputation both grow
at rate $\sqrt N$.

While Theorem~\ref{thm:balanced-optimal} provides optimal
checkpoint sets under uniform and worst-case assumptions,
its results are unlikely to extend to real-world serving systems,
which target average throughput rather than worst-case latency~\citep{yu2022orca,kwon2023pagedattention,zheng2023sglang},
and where the overlap distribution is far from
uniform (Fig.~\ref{fig:overlap-distribution}).

\paragraph{Distribution-aware checkpointing.} For a given overlap distribution
$p_1,\dots,p_N$ and budget $M$, let $\mathrm{dp}[m,j]$ denote the minimum
recomputation cost of serving depths $\{1,\dots,j\}$ with $m$ checkpoints.
If the last checkpoint is placed at position $s \le j$, the cost decomposes into:
\begin{enumerate}[(1)]
  \item the optimal cost of serving $\{1,\dots,s{-}1\}$ with $m{-}1$
  checkpoints, namely $\mathrm{dp}[m{-}1,s{-}1]$;
  \item the cost of using $s$ as the deepest checkpoint for depths
  $t \in \{s,\dots,j\}$, namely $w(s,j)=\sum_{t=s}^j p_t(t{-}s)$
\end{enumerate}

\begin{theorem}[Exact dynamic program]
\label{thm:dp}
For a given overlap depth distribution, the optimal checkpoint placement
satisfies the recurrence 
\[
\mathrm{dp}[0,j] = \sum_{t=1}^{j} p_t\, t,\qquad
\mathrm{dp}[m,j]
=
\min_{1 \le s \le j}
\bigl(
\mathrm{dp}[m-1,s-1] + w(s,j)
\bigr).
\]

With prefix sums $P_j = \sum_{t=1}^j p_t$ and $T_j = \sum_{t=1}^j t\, p_t$,
each DP layer can be evaluated in $O(N)$ time using the monotone convex hull
trick (the slopes $-s$ are monotone in $s$ and the query points $P_j$ are
non-decreasing in $j$), giving $O(NM)$ total time.
\end{theorem}

For a fixed estimated overlap law and budget, this schedule can be
computed offline and refreshed only when the histogram estimate is updated; it
adds no optimization overhead on the serving path itself.

\subsection{Empirical histograms}

In practice, we do not have access to true overlap depth probabilities since they can only be estimated using historical data. {Lemma}~\ref{thm:stability} justifies using the empirical histogram oracle: if the observed overlap histogram converges to the true law in total variation, then the optimal sparse schedule is necessarily close in value.

\begin{lemma}[Stability under distribution misspecification]
\label{thm:stability}
\[
\left|
\mathbb{E}_p[r(T;\mathcal{C})] - \mathbb{E}_q[r(T;\mathcal{C})]
\right|
\le N \|p-q\|_1.
\]
\end{lemma}

Note that the total variation in the right-hand side can be further bounded when
$q$ is the empirical distribution of $p$~\citep{devroye1986combinatorial,mcdiarmid1989method,berend2012empirical}. Theorem~\ref{thm:empirical-oracle} gives a self-contained stationary-learning guarantee for the histogram oracle: under i.i.d.\ overlap, empirical re-estimation becomes near-optimal at the familiar $n^{-1/2}$ rate.

\begin{theorem}[High-probability plug-in guarantee for the empirical histogram oracle]
\label{thm:empirical-oracle}
Let $T_1,\dots,T_n$ be i.i.d.\ overlap samples from a law
$p=(p_1,\dots,p_N)$, let $\hat p^{(n)}$ be the empirical distribution, and let
$\widehat{\mathcal{C}}_n$ be any optimal checkpoint set of size $M$ for
$\hat p^{(n)}$.  Then for every $\delta \in (0,1)$, with probability at least
$1-\delta$,
\[
\|\hat p^{(n)} - p\|_1
\le
\sqrt{\frac{N}{n}} + \sqrt{\frac{2\log(1/\delta)}{n}}.
\]
Consequently, the plug-in schedule satisfies
\[
0 \le \mathbb{E}_p[r(T;\widehat{\mathcal{C}}_n)] - V_M(p)
\le 2N\left(\sqrt{\frac{N}{n}} + \sqrt{\frac{2\log(1/\delta)}{n}}\right),
\]
where $V_M(p)$ is the optimal $M$-checkpoint cost under the true law.
\end{theorem}

We also study the scenario where the true overlap depth distribution can drift
over time. Exponential reweighting~\citep{BarveLong1997,MohriMunoz2012,MazzettoUpfal2023} is a classical technique to track drifting
distributions, where older samples receive exponentially decaying weights
proportional to $\gamma^{t-s}$ so that the reweighted empirical estimate
remains close in total variation to the current distribution. It allows us to
bound total variation and substitute these bounds into {Lemma}~\ref{thm:stability}.

\begin{theorem}[Bias-variance tracking bound for exponential histograms]
\label{thm:drift-tracking}
Let $T_1,\dots,T_t$ be independent overlap samples with time-varying laws
$p_1,\dots,p_t$, and define the exponentially weighted empirical histogram
\[
\hat p_t^{(\gamma)}
=
\frac{1-\gamma}{1-\gamma^t}
\sum_{s=1}^t \gamma^{t-s} e_{T_s},
\qquad 0<\gamma<1,
\]
where $e_{T_s}$ is the one-hot vector at the observed overlap depth.  Then
\[
\mathbb{E}\!\left[\|\hat p_t^{(\gamma)} - p_t\|_1\right]
\le
\frac{1-\gamma}{1-\gamma^t}\sum_{s=1}^t \gamma^{t-s}\|p_s-p_t\|_1
+ \sqrt{
N\frac{1-\gamma}{1+\gamma}\frac{1+\gamma^t}{1-\gamma^t}
}.
\]
\end{theorem}

Combined with Lemma~\ref{thm:stability}, this immediately yields a corresponding excess-recomputation bound for schedules selected from the exponentially weighted histogram.

For large $t$, the variance term behaves as $\sqrt{N(1-\gamma)/(1+\gamma)}$.
Under bounded drift $\|p_{u}-p_{u-1}\|_1 \le \Delta$ for all $u$, the bias
term is at most $\Delta\gamma/(1-\gamma)$.  Setting $\gamma = 0.99$ (our
experimental choice) trades off these terms in favor of low variance while
still tracking moderate drift; empirically we did not observe strong
sensitivity to this choice across our datasets.

\section{Experimental Methodology}
\label{sec:method}

\paragraph{Baselines and Methods.} We compare our method against dense and deterministic sparse checkpoint placement strategies. We additionally include a logarithmic baseline with exponentially increasing gap sizes from the beginning of the prefix; in the figures this baseline is labeled \textsc{Logarithmic}. For hybrid models, we assume that all strategies store the entire KV cache for the full attention layers; they differ only in the location of recurrent layer checkpoints. Since modern serving systems store KV cache in blocks and efficient GPU kernels operate blockwise~\citep{dao2022flashattention,yang2024fla,lefaudeux2022xformers,kwon2023pagedattention}, we clip every checkpoint position to the block boundaries; the block size is $B=128$ tokens throughout. All methods are subjected to this same implementation constraint, so the measured system curves should be read as blockwise execution results rather than as an idealized token-level runtime.

\begin{table}[!t]
\centering
\caption{Recurrent-state checkpoint placement strategies. All hybrid-model runs store the full attention KV cache; the methods differ only in where recurrent-layer checkpoints are placed; $L$ is the cached prefix length, $B$ is the block size, and $M$ is the recurrent-checkpoint budget.}\label{tab:baselines}
\small
\begin{tabular}{l l l}
\toprule
Strategy & Checkpoint positions & Tail recompute \\
\midrule
No cache        & ---                                                          & $L$ \\
KV only         & ---                                                          & $L$ \\
Balanced        & $\lfloor i(L+1)/(M+1)\rfloor,\ i=1,\dots,M$          & $\left\lceil\frac{L+1}{M+1}\right\rceil - 1$ \\
Block           & $B, 2B, \ldots, \lfloor L/B \rfloor B$                                            & $L \bmod B$ \\
$\sqrt{L}$      & $\lfloor\sqrt{L}\rfloor$-spaced                                            & $L \bmod \lfloor\sqrt{L}\rfloor$  \\
\midrule
DP-optimal      & solved via Theorem~\ref{thm:dp}                            & distribution-dependent \\
\bottomrule
\end{tabular}
\normalsize
\end{table}

\paragraph{Our method.} In the main results section we report the exponentially reweighted DP-optimal checkpoint placement with $\gamma=0.99$, re-solving and updating the histogram every 10 requests to tackle distribution shift.

\paragraph{Datasets.} In real-world data, unrelated
requests almost never have long common prefixes, because the number of
possible token sequences grows combinatorially with length. However,
requests can share long non-trivial prefixes if they start with a static chunk of
text. We group such datasets into two kinds:
\begin{enumerate}[(1)]
\item \emph{single document with multiple questions}: QuALITY~\citep{pang2022quality}, NarrativeQA~\citep{kocisky2018narrativeqa};
\item \emph{long shared prompt}: we take real system prompts from major LLM providers~\citep{systempromptleaks2024}, and combine them with real user queries from the ShareGPT~\citep{sharegpt2023} dataset.
\end{enumerate}

All these datasets are structured so that requests come in groups that may have long prefix overlap, and the maximum possible speedup is bounded by its length. We limit each dataset to 1000 requests for prototype wall-clock runs. For savings results, which can be calculated without performing real inference and storing the KV cache, we use 10K requests. To simulate online cache state, we order requests according to a Poisson arrival process; under the fixed last-$K$ policy, only the induced request order affects the results.

\paragraph{KV cache management.} Orthogonal to checkpoint placement efficiency, the maximum speedup from prefix caching is bounded by the cache size and cache eviction policies: even perfectly placed checkpoints do not bring any speedup if the cached checkpoints have already been evicted. To isolate this effect, we use a simple caching strategy that keeps exactly the last $K$ entries. We do not include a direct baseline from~\cite{pan2024marconi} in these plots because Marconi et al. optimize admission and eviction across entries, whereas here we fix a simple last-$K$ policy precisely to isolate within-entry checkpoint placement.

\paragraph{Implementation details.}
All prototype timing experiments were run on an NVIDIA RTX~2080~Super with 8~GB VRAM, paired with an Intel Core~i5-12400F CPU and 16~GB DDR4 RAM, using the HuggingFace Transformers model implementation, CUDA~13.0, and PyTorch~2.10. We use a cache capacity of $K=100$ entries for NarrativeQA (simulation), $K=50$ for System Prompts, and $K=35$ for QuALITY under the fixed last-$K$ policy; these values are chosen to fit the available CPU RAM. We use block size $B=64$, matching the block granularity of the Flash Linear Attention kernels, and an exponentially weighted overlap histogram with $\gamma=0.99$ refreshed every $10$ requests. We use chunked prefill, because when resuming from a non-first position HuggingFace Transformers materializes a full attention mask that does not fit the available VRAM for realistic request lengths. Reported timing curves are obtained from a single deterministic run for each dataset and budget; pilot runs with different random seeds produced negligible variation, so we omit error bars in this prototype evaluation and defer a more thorough statistical treatment to a full-device study on production hardware.

\paragraph{Capturing checkpoints.} We report inference time and capture recurrent-state checkpoints in separate untimed runs, because the current Flash Linear Attention library does not expose a Python API for non-last checkpoint extraction, although the underlying low-level kernels already return recurrent states at block-$64$ positions. We view removing this off-critical-path separation as a straightforward engineering task and defer it, together with the full-device evaluation, to future work.

\paragraph{Overlap distributions.} Figure~\ref{fig:overlap-distribution}
shows the empirical overlap depth distributions across all three datasets. The
distributions are far from uniform: QuALITY has a sharp spike near the full
document length, NarrativeQA exhibits a wide multimodal spread, and System
Prompts concentrate mass at short prefix lengths. These shapes explain why
distribution-aware placement outperforms balanced spacing, especially at low
budgets.

\begin{figure}[!t]
  \centering
  \includegraphics[width=\textwidth]{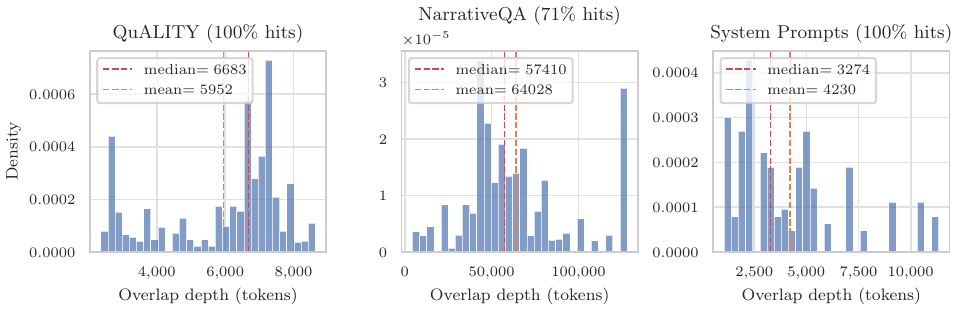}
  \caption{Empirical overlap depth distributions across datasets.}
  \label{fig:overlap-distribution}
\end{figure}

\section{Results}
\label{sec:results}

\paragraph{Token savings.}
Figure~\ref{fig:savings} shows the recurrent-work reduction factor on NarrativeQA, where sequences are too long for the same prototype timing setup on available hardware. The DP-optimal schedule consistently dominates the deterministic baselines across the full budget range. The advantage is largest at low checkpoint budgets, where the empirical overlap distribution is most non-uniform, and narrows as the budget grows and all methods cover an increasing fraction of the prefix.

\begin{figure}[!t]
  \centering
  \includegraphics[width=.9\textwidth]{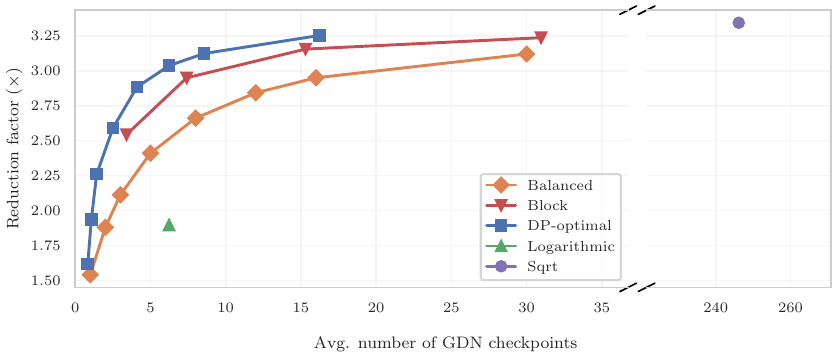}
  \caption{Recurrent-work reduction factor vs.\ checkpoint budget on NarrativeQA.}
  \label{fig:savings}
\end{figure}

\paragraph{Prototype wall-clock time.}
Figure~\ref{fig:speedup} shows prototype wall-clock Pareto fronts on QuALITY and System
Prompts.
Wall-clock time was measured on a representative layer group (1 full attention layer
and 3 GatedDeltaNet layers) of the Qwen-3.5-0.8B model~\citep{qwen2025qwen35}. We therefore interpret these measurements as evidence that token savings translate into runtime savings under a realistic implementation constraint, rather than as full-model serving benchmarks.

\begin{figure}[!t]
  \centering
  \includegraphics[width=\textwidth]{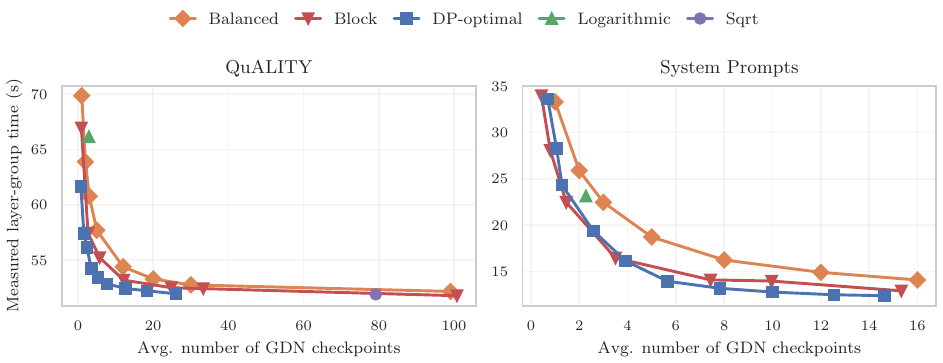}
  \caption{Prototype layer-group Pareto fronts: measured layer-group time vs.\ checkpoint budget on (a) QuALITY and (b) System Prompts.}
  \label{fig:speedup}
\end{figure}

\paragraph{When does distribution-aware placement help?}
The advantage of DP-optimal over balanced is largest at low checkpoint budgets
and on workloads with concentrated overlap distributions. On System Prompts
(Fig.~\ref{fig:speedup}b), where overlap mass concentrates at short prefix
lengths (Fig.~\ref{fig:overlap-distribution}c), the DP places checkpoints
densely near the divergence peak and leaves the tail uncovered, achieving
the same measured runtime as block caching but with fewer checkpoints. On QuALITY
(Fig.~\ref{fig:speedup}a), the gap between DP-optimal and balanced is smaller
because the overlap distribution is closer to a point mass at the full
document length, reducing the scope for non-uniform placement. In the limit of
large budgets, all strategies converge because every prefix position is
approximately covered.

\begin{figure}[!t]
  \centering
  \includegraphics[width=\textwidth]{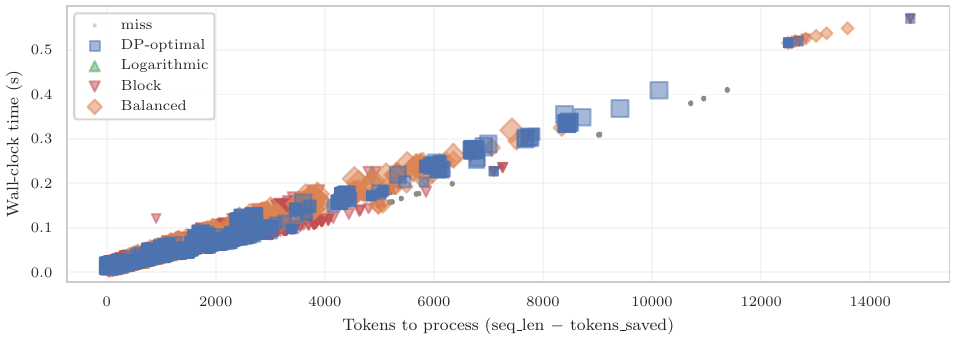}
  \caption{Real wall-time vs amount of recurrent work on System Prompts dataset.
  Marker size corresponds to amount of cache for CPU$\to$GPU transfer.}
  \label{fig:time-per-token}
\end{figure}

For NarrativeQA, recall that we report the reduction factor rather than wall-clock time because the sequences do not fit the prototype setup. We verify in Fig.~\ref{fig:time-per-token} that the number of tokens to recompute is essentially proportional to wall-clock time on the smaller datasets, so the reduction-factor curves should translate directly to runtime savings on hardware that does fit NarrativeQA.

\paragraph{Overhead decomposition.}
The wall-clock cost of a cache hit consists of:
\begin{inparaenum}[(i)]
\item loading the stored checkpoint from CPU to GPU,
\item replaying the suffix through the recurrent layers, and
\item for the attention layers, loading the cached KV state. 
\end{inparaenum}

Figure~\ref{fig:time-per-token} shows that wall-clock time is approximately linear in the number of tokens to recompute, with a small constant offset due to cache loading. This offset is visible as the nonzero intercept for the hit strategies; its magnitude grows with the amount of cached state (marker size in the figure). For the checkpoint budgets used in our experiments, the loading overhead is small relative to the recomputation savings, which is why the prototype Pareto fronts show net runtime improvement.

\section{Conclusion}
\label{sec:conclusion}

Sparse prefix caching is based on a structural asymmetry in hybrid and recurrent LLMs: exact reuse does not require dense per-token recurrent state. By storing only a sparse set of exact checkpoints and replaying the missing suffix, we obtain a new memory/recompute trade-off that cannot be obtained in dense KV-centric serving.

The main limitation is that our approach is most useful when many requests share a substantial but not identical prefix inside a retained cache entry; in some situations this is indeed the case, but append-only chat workloads have different properties: if you include all history every time and only append to it, caching the last checkpoint is obviously sufficient.

For future work, we first of all note engineering challenges: to implement caching in a useful way one has to implement efficient state extraction and restoration in production runtimes and combine checkpoint placement with Marconi-style admission/eviction. As for theory, we note that our single-prefix theory can hopefully be further extended to general prefix trees with branching and eviction.

\bibliographystyle{plainnat}
\bibliography{paper}

\appendix
\section{Proofs}
\label{app:proofs}

All theorem and lemma statements in this appendix appear in the main text; we only restate them here to the extent needed for the proofs, and do not repeat this notice for individual results.

\subsection{Balanced sparse checkpoints}
\begin{lemma}[Gap decomposition under uniform overlap]
\label{lem:uniform-gap}
If $p_t = 1/N$ for $t=1,\dots,N$, then
\begin{equation}
N\,\mathbb{E}[r(T;\mathcal{C})]
= \sum_{i=0}^{M} \sum_{h=0}^{g_i-1} h
= \sum_{i=0}^{M} \frac{g_i(g_i-1)}{2}.
\label{eq:uniform-gap-objective}
\end{equation}
\end{lemma}

\begin{proof}[Proof of Lemma~\ref{lem:uniform-gap}]
On the gap $(c_i,c_{i+1})$, the reusable depth remains fixed at $c_i$, so the
recompute costs on positions $t=c_i,c_i{+}1,\dots,c_{i+1}{-}1$ are precisely
$0,1,\dots,g_i{-}1$.  Under the uniform law $p_t=1/N$, every depth contributes
equally, hence (\ref{eq:uniform-gap-objective}).
\end{proof}

\begin{proof}[Proof of Theorem~\ref{thm:balanced-optimal}.1]
 We will show that a checkpoint set $\mathcal{C}$ minimizes
$\mathbb{E}[r(T;\mathcal{C})]$ $\Leftrightarrow$ its induced gap lengths satisfy $g_i \in \{q,q+1\}$ for all $i$, with
exactly $\rho$ gaps equal to $q+1$. In particular, any evenly spaced schedule such as
\[
c_i = \left\lfloor \frac{i(N+1)}{M+1} \right\rfloor, \qquad i=1,\dots,M,
\]
is optimal.

$\Rightarrow$ By Lemma~\ref{lem:uniform-gap}, minimizing $\mathbb{E}[r(T;\mathcal{C})]$ is
equivalent to minimizing
\[
\sum_{i=0}^{M} \phi(g_i),
\qquad
\phi(x)=\frac{x(x-1)}{2},
\]
over positive integers $g_i$ summing to $N{+}1$.  The function $\phi$ is
strictly convex.

Suppose there exist indices $i,j$ with $g_i \ge g_j{+}2$.  Replacing
$(g_i,g_j)$ by $(g_i{-}1,g_j{+}1)$ preserves the total sum and changes the
objective by
\[
\phi(g_i)+\phi(g_j)-\phi(g_i{-}1)-\phi(g_j{+}1)
= g_i-g_j-1 > 0.
\]
Hence any feasible gap vector with two gaps differing by at least~2 can be
strictly improved.  Therefore every optimum must have all gaps differing by at
most one.

$\Leftarrow$ Now let $K=M+1$ and write
\[
N+1 = qK + \rho, \qquad 0 \le \rho < K.
\]
The only
positive integers summing to $N{+}1$ and differing pairwise by at most one are
exactly $K{-}\rho$ copies of $q$ and $\rho$ copies of $q{+}1$.
Conversely, $\phi$ is permutation-invariant and
every such gap multiset has the same objective value, so every such schedule is
optimal.

Finally, $c_i=\lfloor i(N{+}1)/(M{+}1)\rfloor$ induces gaps differing by at
most one, hence it is optimal.

The exact optimal cost follows by substituting this gap multiset into Lemma~\ref{lem:uniform-gap}:
\[
N\,\mathbb{E}[r(T;\mathcal{C}^\star)]
=
(K-\rho)\frac{q(q-1)}{2}
+ \rho\frac{(q+1)q}{2}.
\]
Since $q = \lfloor (N{+}1)/(M{+}1)\rfloor = N/(M{+}1) + O(1)$, dividing by $N$ gives $\mathbb{E}[r(T;\mathcal{C}^\star)] = N/(2(M{+}1)) + O(1)$ as $N\to\infty$ with $M$ fixed.
\end{proof}

\begin{proof}[Proof of Theorem~\ref{thm:balanced-optimal}.2]

Let
\[
R_\infty(\mathcal{C}) = \max_{1 \le t \le N} r(t;\mathcal{C}).
\]
We will show that 
\[
\min_{|\mathcal{C}|=M} R_\infty(\mathcal{C})
=
\left\lceil \frac{N+1}{M+1} \right\rceil - 1.
\]
and balanced schedules reach the minima for every $M$.

The maximum recomputation equals the largest gap minus one:
$R_\infty(\mathcal{C}) = \max_{0\le i\le M} (g_i-1)$.
Since the $M{+}1$ positive integers $g_0,\dots,g_M$ sum to $N{+}1$, the
pigeonhole principle gives
\[
\max_i g_i \ge \left\lceil \frac{N+1}{M+1}\right\rceil,
\]
and therefore
$R_\infty(\mathcal{C}) \ge \lceil (N{+}1)/(M{+}1)\rceil - 1$.

For the upper bound, take any balanced schedule from
Theorem~\ref{thm:balanced-optimal}.  All gaps are either $q$ or $q{+}1$, so
$R_\infty(\mathcal{C}) = \lceil (N{+}1)/(M{+}1)\rceil - 1$, and the lower
bound is attained.

\end{proof}

\subsection{Histogram-based oracle}

\begin{proof}[Proof of Lemma~\ref{thm:stability}]
For any fixed checkpoint set $\mathcal{C}$,
\[
\left|
\mathbb{E}_p[r(T;\mathcal{C})]-\mathbb{E}_q[r(T;\mathcal{C})]
\right|
=
\left|
\sum_{t=1}^N (p_t-q_t)\,r(t;\mathcal{C})
\right|
\le
\sum_{t=1}^N |p_t-q_t|\, r(t;\mathcal{C})
\le
N \|p{-}q\|_1,
\]
since $0 \le r(t;\mathcal{C}) \le N$ for every $t$.
Optimizing over $\mathcal{C}$ on both sides and using a standard
swap argument (as in the proof of Theorem~\ref{thm:empirical-oracle} below)
extends the bound to the value function $V_M$.
\end{proof}

\begin{proof}[Proof of Theorem~\ref{thm:empirical-oracle}]
Let
\[
F_n = \|\hat p^{(n)} - p\|_1.
\]
By the Cauchy--Schwarz inequality,
\[
F_n \le \sqrt{N}\,\|\hat p^{(n)} - p\|_2.
\]
Taking expectations and applying Jensen's inequality,
\[
\mathbb{E}[F_n]
\le
\sqrt{N\,\mathbb{E}\!\left[\|\hat p^{(n)} - p\|_2^2\right]}.
\]
Now
\[
\mathbb{E}\!\left[\|\hat p^{(n)} - p\|_2^2\right]
=
\sum_{t=1}^N \mathrm{Var}(\hat p^{(n)}_t)
=
\frac{1}{n}\sum_{t=1}^N p_t(1-p_t)
\le \frac{1}{n},
\]
so
\[
\mathbb{E}[F_n] \le \sqrt{\frac{N}{n}}.
\]

Next, changing one sample $T_i$ modifies the empirical histogram by moving mass
$1/n$ from one coordinate to another, so $F_n$ changes by at most $2/n$.
McDiarmid's inequality therefore gives
\[
\Pr\!\left[
F_n - \mathbb{E}[F_n] \ge \varepsilon
\right]
\le
\exp\!\left(-\frac{n\varepsilon^2}{2}\right).
\]
Setting $\varepsilon = \sqrt{2\log(1/\delta)/n}$ yields, with probability at
least $1-\delta$,
\[
F_n
\le
\sqrt{\frac{N}{n}} + \sqrt{\frac{2\log(1/\delta)}{n}}.
\]

For the plug-in bound, let $\widehat{\mathcal{C}}_n$ be optimal under
$\hat p^{(n)}$.  Then
\begin{align*}
\mathbb{E}_p[r(T;\widehat{\mathcal{C}}_n)] - V_M(p)
&\le
\bigl|\mathbb{E}_p[r(T;\widehat{\mathcal{C}}_n)] - \mathbb{E}_{\hat p^{(n)}}[r(T;\widehat{\mathcal{C}}_n)]\bigr|
+ \bigl|V_M(\hat p^{(n)}) - V_M(p)\bigr| \\
&\le N F_n + N F_n = 2N F_n,
\end{align*}
using Lemma~\ref{thm:stability} for both terms.
\end{proof}

\subsection{Tracking Under Drift}

\begin{proof}[Proof of Theorem~\ref{thm:drift-tracking}]
Write
\[
w_{t,s}=\frac{(1-\gamma)\gamma^{t-s}}{1-\gamma^t},
\qquad
\hat p_t^{(\gamma)}=\sum_{s=1}^t w_{t,s} e_{T_s},
\qquad
\bar p_t^{(\gamma)}=\mathbb{E}[\hat p_t^{(\gamma)}]
=
\sum_{s=1}^t w_{t,s} p_s.
\]
Then by the triangle inequality,
\[
\mathbb{E}\!\left[\|\hat p_t^{(\gamma)}-p_t\|_1\right]
\le
\|\bar p_t^{(\gamma)}-p_t\|_1
+ \mathbb{E}\!\left[\|\hat p_t^{(\gamma)}-\bar p_t^{(\gamma)}\|_1\right].
\]
The bias term is
\[
\|\bar p_t^{(\gamma)}-p_t\|_1
=
\left\|
\sum_{s=1}^t w_{t,s}(p_s-p_t)
\right\|_1
\le
\sum_{s=1}^t w_{t,s}\|p_s-p_t\|_1.
\]

For the stochastic term, Jensen and Cauchy--Schwarz give
\[
\mathbb{E}\!\left[\|\hat p_t^{(\gamma)}-\bar p_t^{(\gamma)}\|_1\right]
\le
\sqrt{
N\,\mathbb{E}\!\left[\|\hat p_t^{(\gamma)}-\bar p_t^{(\gamma)}\|_2^2\right]
}.
\]
Because the samples are independent and centered cross-terms vanish,
\[
\mathbb{E}\!\left[\|\hat p_t^{(\gamma)}-\bar p_t^{(\gamma)}\|_2^2\right]
=
\sum_{s=1}^t w_{t,s}^2\,
\mathbb{E}\!\left[\|e_{T_s}-p_s\|_2^2\right].
\]
Now
\[
\mathbb{E}\!\left[\|e_{T_s}-p_s\|_2^2\right]
=
1-\|p_s\|_2^2 \le 1,
\]
so
\[
\mathbb{E}\!\left[\|\hat p_t^{(\gamma)}-\bar p_t^{(\gamma)}\|_2^2\right]
\le
\sum_{s=1}^t w_{t,s}^2.
\]
Finally,
\[
\sum_{s=1}^t w_{t,s}^2
=
\frac{(1-\gamma)^2}{(1-\gamma^t)^2}\sum_{k=0}^{t-1}\gamma^{2k}
=
\frac{(1-\gamma)^2}{(1-\gamma^t)^2}\cdot\frac{1-\gamma^{2t}}{1-\gamma^2}
=
\frac{1-\gamma}{1+\gamma}\cdot\frac{1+\gamma^t}{1-\gamma^t}.
\]
Combining the bias and variance bounds proves the claim.
\end{proof}

\subsection{Proof of Theorem~\ref{thm:dp}}
\begin{proof}
The recurrence follows from the cost decomposition in Section~\ref{sec:theory}.
For the $O(NM)$ complexity, write $w(s,j) = (T_j - T_{s-1}) - s(P_j - P_{s-1})$.
Then
\[
\mathrm{dp}[m,j] = T_j + \min_{1 \le s \le j}
\bigl(\mathrm{dp}[m{-}1,s{-}1] - T_{s-1} + s P_{s-1} - s P_j\bigr).
\]
For fixed $m$, each candidate $s$ contributes a line $y = (-s)x +
(\mathrm{dp}[m{-}1,s{-}1] - T_{s-1} + s P_{s-1})$ evaluated at query point
$x = P_j$.  The slopes $-s$ are strictly decreasing in $s$, and the query
points $P_j$ are non-decreasing in $j$ (since $p_t \ge 0$). A monotone convex
hull trick therefore evaluates each layer in $O(N)$ amortized time, giving
$O(NM)$ total.

\emph{Note:} if some $p_t = 0$, consecutive query points may coincide, but the
monotone-pointer technique handles ties without affecting the amortized
bound.
\end{proof}

\section{Additional Figures}
\label{app:figures}

\begin{figure}[H]
  \centering
  \includegraphics[width=\textwidth]{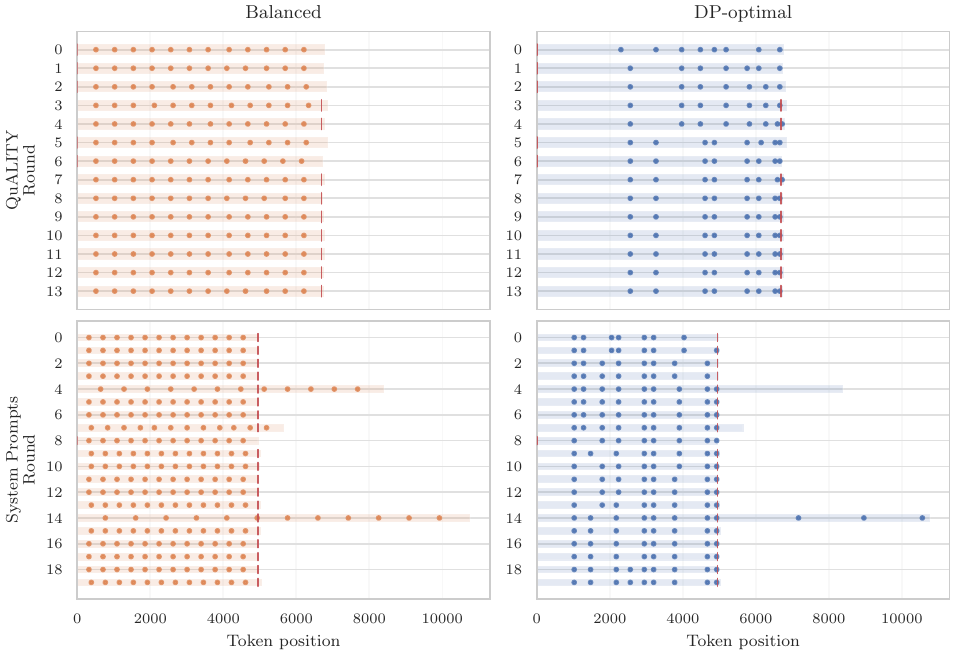}
	\caption{Checkpoint placements for sampled conversations (a) QuALITY,
	(b) System Prompts }
  \label{fig:checkpoint-placement}
\end{figure}

\end{document}